%% file: FINAL_VERSION.tex
\documentclass{article}

\usepackage{amssymb}
\usepackage{algorithm}
\usepackage{amsmath}
\usepackage{amsthm}
\usepackage{graphicx}
\usepackage{subfigure}
\usepackage{tabularx}
\usepackage{booktabs}
\usepackage{bbm}
\usepackage[noend]{algpseudocode}
\usepackage{array}
\usepackage{balance}
\usepackage{multirow}
\usepackage{multicol}
\usepackage{threeparttable}
\usepackage{cite}
\usepackage{fullpage}

\allowdisplaybreaks

\input{yaweinewcomm}

\newtheorem{Theorem}{\bf{Theorem}}
\newtheorem{Corollary}{\bf{Corollary}}
\newtheorem{Lemma}{\bf{Lemma}}
\newtheorem{Remark}{\bf{Remark}}
\newtheorem{Assumption}{\bf{Assumption}}

\begin{document}

\title{Proximal Online Gradient is Optimum for Dynamic Regret: A General Lower Bound}

\author{$^1$Yawei Zhao\footnote{ represents equal contribution.},
        $^2$Shuang Qiu$^{\ast}$,
        $^3$Ji Liu \\
$^1$National University of Defense Technology, Changsha, China. \\  
$^2$University of Michigan, Ann Arbor, USA. \\
$^3$Uiversity of Rochester, Rochester, USA. \\
\texttt{zhaoyawei@nudt.edu.cn}, \texttt{qiush@umich.edu}, \texttt{ji.liu.uwisc@gmail.com}.
}

\date{}

\maketitle

\begin{abstract}
In online learning, the \emph{dynamic} regret metric chooses the reference (optimal) solution that may change over time, while the typical (static) regret metric assumes the reference solution to be constant over the whole time horizon. The dynamic regret metric is particularly interesting for applications such as online recommendation (since the customers' preference always evolves over time). While the online gradient method has been shown to be optimal for the static regret metric, the optimal algorithm for the dynamic regret remains unknown. In this paper, we show that proximal online gradient (a general version of online gradient) is optimum to the dynamic regret by showing that the proved lower bound matches the upper bound. It is highlighted that we provide a new and general lower bound of dynamic regret. It provides new understanding about the difficulty to follow the dynamics in the online setting.  
\end{abstract}

\section{Introduction}

Online learning \cite{Zinkevich:2003,ShalevShwartz:2012dz,Hazan2016Introduction,pmlr-v51-mohri16,Zhang:2018tu,pmlr-v54-jun17a,pmlr-v23-jain12,pmlr-v70-zhang17g} is a hot research topic for the last decade of years, due to its application in practices such as online recommendation \cite{pmlr-v51-chaudhuri16},  online collaborative filtering \cite{Liu2017,Awerbuch:2007:OCF}, moving object detection \cite{Nair:2004} and many others, as well as its close connection with other research areas such as stochastic optimization \cite{DBLP:rakhlin,pmlr-v84-liu18a}, image retrieval \cite{Gao:2017:SOL}, multiple kernel learning \cite{JMLR:v17:14-148,pmlr-v84-shen18a}, and bandit problems \cite{Flaxman:2005um,Arora:2012ta,pmlr-v54-kwon17a,pmlr-v51-kocak16}, etc.

The typical objective function in online learning is to minimize the (static) regret defined below
\begin{align}
\sum_{t=1}^T f_t(\x_t) - \underbrace{\min_{\x\in\Xcal}\sum_{t=1}^T f_t(\x)}_{\text{the optimal reference}}, \label{eq:oldregret}
\end{align}
where $\x_t$ is the decision made at step $t$ after receiving the information before that (e.g., $\{\nabla f_s(\x_s), f_s(\x_s)\}_{s=1}^{t-1}$). The optimal reference is chosen at the point that minimizes the sum of all component functions up to time $T$. However, the way to decide the optimal reference may not fit some important applications in practice. For example, in the recommendation task, $f_t(\x)$ is the regret at time $t$ decided by the $t$-th coming customer and our recommendation strategy $\x$.  Based on the definition of regret in \eqref{eq:oldregret}, it implicitly assumes that the optimal recommendation strategy is constant over time, which is not necessarily true for the recommendation task (as well as many other applications) since the costumers' preference usually evolves over time.

\cite{Zinkevich:2003} proposed to use the \emph{dynamic} regret as the metric for online learning, that allows the optimal strategy changing over time. More specifically, it is defined by
\begin{align}
\Rcal_{T}^{A} := & \sum_{t=1}^T f_t(\x_t) - \min_{\{\y_t\}_{t=1}^T \in \mathcal{L}_{D_{0}}^T } \sum_{t=1}^T f_t(\y_t), \label{eq:dregret}
\end{align} where $A$ denotes the algorithm that decides $\x_t$ iteratively, $\{\y_t\}_{t=1}^T$ is short for a sequence $\{\y_1, \y_2, \cdots, \y_T\}$, and the dynamics upper bound $\mathcal{L}_{D_{0}}^T$ is defined by
\begin{align}
\mathcal{L}_{D_{0}}^T := \left \{\{\y_t\}_{t=1}^T : \sum_{t=1}^{T-1} \lrnorm{ \y_{t+1} - \y_t} \le D_0 \right \}.  \label{eq:D0}
\end{align}
It was shown that the dynamic regret of Online Gradient (OG) is bounded \cite{Hall:2013vr,Hall:2015ct,Zinkevich:2003} by
\begin{align}
\Rcal_T^{\textsc{OG}} \lesssim \sqrt{T} + \sqrt{T}D_0. \label{eq:ogdbound}
\end{align}
where $\lesssim$ means ``less than equal up to a constant factor''. 
This reminds people to ask a few fundamental questions: 
\begin{itemize}
\item As we know the dependence on $T$ is tight, since OG is optimum for static regret. But, is the dependence to the dynamics $D_0$ tight? In other words, Is OG also optimal for dynamic regret?
\item Is this bound tight enough? If no, how to design a ``smarter'' algorithm to follow the dynamics?
\item How difficult to follow dynamics in online learning?
\end{itemize}
Although the dynamic regret receives more and more attention recently \cite{Mokhtari:2016jz,Yang:2016ud,Zhang:2016wl,Shahrampour:2018dm,Hall:2015ct,Jadbabaie:2015wg} and some successive studies claim to improve this result by considering specific functions types (e.g., strongly convex $f_t$), or considering different definitions of dynamic regret, these fundamental questions still remain unsolved. 

In this paper, we consider a more general setup for the problem
\begin{equation}
f_t(\x) = F_t(\x) + H(\x), 
\label{eq:f_t}
\end{equation}
with $F_t(\x)$ and $H(\x)$ being only convex and closed, and a more general definition for dynamic constraint in \eqref{eq:Dbeta}
\begin{align}
\mathcal{L}_{D_{\beta}}^T := \left \{\{\y_t\}_{t=1}^T : \sum_{t=1}^{T-1} t^{\beta}\cdot \lrnorm{\y_{t+1} - \y_t} \le D_{\beta} \right \}.  \label{eq:Dbeta}
\end{align}
where $\beta$ and $D_{\beta}$ are the pre-defined parameters to restrict the change of reference models over time. We show that the upper bound of the Proximal Online Gradient (POG) algorithm, can achieve
\begin{align}
\Rcal_T^{\textsc{POG}} \lesssim \sqrt{T} + \sqrt{T^{1-\beta}\cdot D_{\beta}}. \label{eq:POGbound}
\end{align} When $\beta = 0$ and $H(\x)\equiv0$, \eqref{eq:POGbound} recovers the early result in \eqref{eq:ogdbound}. But, \eqref{eq:POGbound} still holds for proximal mapping when updating $\x_t$. When $\beta > 0$, since $D_{\beta} < D_0 T^{\beta}$, \eqref{eq:POGbound} is slightly better than the proved special case in \eqref{eq:ogdbound}.


To understand the difficulty of following dynamics in online learning, we derive the lower bound (that measures the dynamic regret by the \emph{optimal} algorithm) and show that the proved upper bound for POG matches the lower bound up to a constant factor, which indicates POG is an optimal algorithm even for dynamic regret (not just for static regret).

\section{Related work}
In this section, we outline and review the existing work about online learning problem with the regret in static and dynamic environments briefly.
\subsection{Static Regret}
Online gradient in the static environment has been extensively investigated for the last decade of years \cite{ShalevShwartz:2012dz,Hazan2016Introduction,Duchi:2011}. Specifically, when $f_t$ is strongly convex, the regret of online gradient is $\Ocal{\log T}$. When $f_t(\cdot)$ is only convex, the regret of online gradient is $\Ocal{\sqrt{T}}$. 

\subsection{Dynamic Regret}
\cite{Zinkevich:2003} obtains the regret in the order of $\Ocal{\sqrt{T}D_0+\sqrt{T}}$ for the convex function $f_t$. Similarly, assume the dynamic constraint is defined by $\sum_{t=1}^{T-1} \lrnorm{\y_{t+1} - \Phi(\y_t)}\le D_0$, where $\Phi(\cdot)$ provides the prediction about the dynamic environment. When $\Phi(\y_t)$ predict the dynamic environment accurately, \cite{Hall:2013vr,Hall:2015ct} obtain a better regret than \cite{Zinkevich:2003}, but it is still bounded by $\Ocal{\sqrt{T}D_0+\sqrt{T}}$.

Additionally, assume $f_t$ is $\alpha$ strongly convex and $\beta$ smooth, and the dynamic constraint is defined by $D^{\ast} := \sum_{t=1}^{T-1}\lrnorm{\y_{t+1}^\ast - \y_t^\ast},\text{~where~} \y_t^\ast := \argmin_{\y\in\Xcal} f_t(\y)$. \cite{Mokhtari:2016jz} obtains $\Ocal{D^{\ast}}$ regret.  When querying noisy gradient, \cite{Bedi:2018ef} obtains $\Ocal{D^{\ast} + \mathcal{E}}$ regret, where $\mathcal{E}$ is the cumulative gradient error. \cite{Yang:2016ud,pmlr-v84-gao18a} extend it for non-strongly convex and non-convex functions, respectively. \cite{Shahrampour:2018dm} extends it to the decentrialized setting\footnote{The definition of $D^{\ast}$ is changed slightly in the decentrialized setting.}.  Furthermore, define $S^{\ast} := \sum_{t=1}^{T-1}\lrnorm{\y_{t+1}^\ast - \y_t^\ast}^2, \text{~where~} \y_t^\ast := \argmin_{\y\in\Xcal} f_t(\y)$. When querying $\Ocal{\kappa}$ with $\kappa := \frac{\beta}{\alpha}$ gradients for every iteration, \cite{Zhang:2016wl} improves the dynamic regret to be $\Ocal{\min\{D^{\ast}, S^{\ast}\}}$. Comparing with the previous work, we obtain a tight regret, and our analysis does not assume the smoothness and strong convexity of $f_t$.   

Other regularities including the functional variation \cite{pmlr-v48-jenatton16,Zhu:2015tr,Besbes:2015gb,Zhang:2018tu}, the gradient variation \cite{Chiang2012Online}, and the mixed regularity \cite{Jadbabaie:2015wg,Chen:2017tt,Jenatton:2016wp} have been investigated to bound the dynamic regret. Those different regularities cannot be compared directly because that they measure different aspects of the variation in the dynamic environment. In the paper, we use \eqref{eq:Dbeta} to bound the regret, and it is the future work to extend our analysis to other regularities.

\cite{Gyorgy:2016} studies a dynamic regret\footnote{It is called shifting regret in \cite{Gyorgy:2016}. To avoid the confusion with many papers that will be discussed in the following subsection, the shifting regret in this paper is defined in a different way from \cite{Gyorgy:2016}.} in a slightly more general setting than \eqref{eq:D0} by relaxing the distance metric $\|\y_{t+1} - \y_t\|$ to a general $\ell_p$ norm $\|\y_{t+1} - \y_t\|_p$ with $p\in (1, 2]$. They obtain an upper bound $\Ocal{\sqrt{D_0T+T}}$ for an algorithm namely \textit{TMD}. This result is essentially consistent with our upper bound, but we consider a different algorithm and a different generalization of the dynamic regret definition, and provide a lower bound more importantly.

Recently, \cite{NeurIPS:2018:ZhangB} provides a lower bound for the case of $\beta=0$ in \eqref{eq:Dbeta}. Comparing with the known result, our lower bound holds for $0\le \beta <1$, and thus is more general. As far as we know, it is the first lower bound for the dynamic regret in the case of $\beta>0$. Besides, the previous result only holds for smooth $f_t$, but our lower bound still holds for non-smooth $f_t$. \cite{NeurIPS:2018:ZhangB} also provides an optimal online method. But, the method is limited to work in the \textit{expert} setting, and requires the smoothness of $f_t$. Our proposed online methods does not have those limitations.

\subsection{Shifting regret (or tracking regret)}
The $M-$shifting regret of an algorithm $A\in\Acal$ is defined by \cite{Herbster1998,Gyorgy:2005wo,Gyorgy:2012wa,Gyorgy:2016,Mourtada:2017vn,JMLR:v17:13-533,NIPS2016_6536,cesabianchi:hal,pmlr-v84-mohri18a,pmlr-v54-jun17a}
\begin{align}
\label{equa_definition_shift_regret}
\widetilde{\Rcal}^{A}_T := \sum_{t=1}^T f_t(\x_t) - \min_{\{\y_t\}_{t=1}^T\in \Lcal^T_{M } }\sum_{t=1}^T f_t(\y_t), 
\end{align} where $\Lcal^T_{M }  =  \left\{ \{\y_t\}_{t=1}^T : \sum_{t=1}^{T-1} \mathbbm{1}\{\y_{t+1} \neq \y_t \} \le  M  \right\}$.
Here, the dynamics is modeled by the number of changes of the reference sequence $\{\y_t\}_{t=1}^T$. 
The shifting regret is closely related to the dynamic regret and can be considered as a variation of dynamic regret, and is usually studied in the setting of learning with expert advice. The result in \cite{Gyorgy:2012wa,Daniely:2015} implies an upper bound $\Ocal{\sqrt{MT\log^2 T}}$ for the shifting regret. The results in both \cite{Luo:2015ud} and \cite{pmlr-v54-jun17a} imply an improved upper bound to $\Ocal{\sqrt{MT\log T}}$.

\section{Notations and Assumptions}
In this section, we introduce notations and important assumptions for the online learning algorithm used throughout this paper.
\subsection{Notations}
Throughout this paper, we use the following notations.
\begin{itemize}
\item $\Acal$ represents the family of all possible online algorithms.

\item $\Fcal$ represents the family of loss functions available to the adversary, where for any loss function $f_t \in\Fcal: \Xcal\subset \RR^d \mapsto \RR$, $f_t(\x) = F_t(\x) + H(\x)$ satisfies Assumption \ref{assumption_convex_F_H} and Assumption \ref{assumption_bound_gradient_domain}. $\Fcal^T$ denotes the function product space by $\underbrace{\Fcal \times \Fcal \times \cdots \times \Fcal}_{T\ \text{times}}$.

\item $\{\u_t\}_{t=1}^T$ represents a sequence of $T$ vectors, namely, $\{\u_1, \u_2, ..., \u_T\}$. $\{f_t\}_{t=1}^T$ denotes a sequence of $T$ functions, which is $\{f_1, f_2, \cdots, f_T \}$.

\item $\Rcal_{T}^{A}$ is the regret for a loss function sequence $\{f_t\}_{t=1}^T \in \Fcal^T$ with a learning algorithm $A \in \Acal$ where $A$ can be POG or OG.

\item $\lrnorm{\cdot}_p$ denotes the $\ell_p$ norm. $\lrnorm{\cdot}$ represents the $\ell_2$ norm by default.

\item $\lesssim$ means ``less than equal up to a constant factor'', and $\gtrsim$ means ``greater than equal up to a constant factor''. $\partial$ represents the subgradient operator. $\EE$ represents the mathematical expectation.

\end{itemize}

\subsection{Assumptions}
We use the following assumptions to analyze the regret of the online gradient.
\begin{Assumption}
\label{assumption_convex_F_H}
Functions $F_t: \Xcal \subset \RR^d \mapsto \RR$ for all $t \in [T]$ and $H: \Xcal \subset \RR^d \mapsto \RR$ are convex and closed but possibly nonsmooth. Particularly, $f_t \in \Fcal$ is defined as $ f_t(\x) = F_t(\x) + H(\x)$.
\end{Assumption}

\begin{Assumption}
\label{assumption_bound_gradient_domain}
The convex compact set $\Xcal$ is the domain for $F_t$ and $H$, and $\|\x -\y \|^2 \leq R$ for any $\x,\y \in \Xcal$. Besides, for any $\x\in\Xcal$ and function $F_t$, $\lrnorm{G_t(\x)}^2 \le G$, where $G_t(\x) \in \partial F_t(\x)$.
\end{Assumption}

\begin{algorithm}[!t]
   \caption{POG: Proximal Online Gradient.}
   \label{algo_pog}
   \begin{algorithmic}[1]
   \Require The learning rate $\eta_t$ with $1\le t\le T$.
       \For {$t=1,2, ..., T$}
            \State Predict $\x_t$.
            \State  Observe the loss function $f_t$ with $F_t$ and $H$, and suffer loss $f_t(\x_t) = F_t(\x_t) + H(\x_t)$.
            \State Query subgradient $G_t(\x_t) \in \partial F_t(\x_t)$.  
            \State $\x_{t+1} = \mathrm{\mathbf{prox}}_{H,\eta_t}(\x_t - \eta_t G_t(\x_t)).$ 
       \EndFor
       \State \textbf{return} $\x_{T+1}$
   \end{algorithmic}
\end{algorithm}

\section{Algorithm}

We use the proximal online gradient (POG) for solving the online learning problem with $f_t(\cdot)$ in the form of \eqref{eq:f_t}. The POG algorithm is a general version of OG for taking care of the regularizer component $H(\cdot)$ in $f_t(\cdot)$. The complete POG algorithm is presented in Algorithm~\ref{algo_pog}. Line 4 of Algorithm~\ref{algo_pog} is the proximal gradient descent step defined by
\begin{align*}
    \x_{t+1} = \mathrm{\mathbf{prox}}_{H,\eta_t}(\x_t - \eta_t G_t(\x_t)),
\end{align*}
where the proximal operator is defined as 
\begin{align*}
    \mathrm{\mathbf{prox}}_{H,\eta_t} (\x') := \arg \min_{\x \in \Xcal} \left \{ H(\x) + \frac{1}{2\eta_t} \|\x-\x'\|^2  \right \}.
\end{align*}
Therefore, the update of $\x_{t+1}$ is also equivalent to 
\begin{align*}
\x_{t+1} = & \mathrm{\mathbf{prox}}_{H,\eta_t}(\x_t - \eta_t G_t(\x_t)) = \arg \min_{\x\in\Xcal} \lrangle{G_t(\x_t), \x} + \frac{1}{2\eta_t} \lrnorm{\x -\x_t}^2 + H(\x).
\end{align*}
The POG algorithm reduces to the OG algorithm when $H(\cdot)$ is a constant function.


\section{Theoretical results}

Recall that we now consider an online learning problem with a dynamic constraint
\begin{align}
\nonumber
\mathcal{L}_{D_{\beta}}^T := \left \{\{\y_t\}_{t=1}^T : \sum_{t=1}^{T-1} t^{\beta}\cdot \lrnorm{\y_{t+1} - \y_t} \le D_{\beta} \right \}, 
\end{align} 
which is more general comparing with the previous definition of the dynamic constraint 
$\mathcal{L}_{D_{0}}^T$ defined in \eqref{eq:D0}.

When $\beta = 0$, $\mathcal{L}_{D_{\beta}}^T$ reduces to the previous definition of the dynamic constraint. Comparing with the previous definition, when $\beta\ge 0$, $D_{\beta}$ allocates larger weights for the future parts of the dynamics than the previous parts. 

\begin{Remark}
It is worth noting that $D_{\beta}$ is a pre-defined parameter to restrict the change of reference models.
\end{Remark}

In this section, we first present an lower bound which was not well studied in previous literature to our best knowledge. Then, we prove an upper bound for the regret based on our general dynamic constraints via proximal online gradient, which holds for a general dynamic regret, instead of $\beta=0$ shown in previous work. We will show that our proved upper bound matches the lower bound, implying the optimality of proximal online gradient algorithm.

\subsection{General lower bound for online convex optimization}
Once we obtain the upper bound for dynamic regret via POG, namely $\sup_{\{f_t\}_{t=1}^T\in\Fcal^T} \Rcal_T^{\textsc{POG}}$, there still remains a question, whether our upper bound's dependency on $D_{\beta}$ and $T$ is tight enough or even optimal. 

Unfortunately, to our best knowledge, this question has not been fully investigated in any existing literature, even for the case of the dynamic regret defined with $D_0$.

To answer this question, we attempt to explore the value of $\sup_{\{f_t\}_{t=1}^T\in\Fcal^T} \Rcal_T^{\textsc{A}}$ for the optimal algorithm $A \in \Acal$, which is formally written as $\inf_{A\in\Acal} \sup_{\{f_t\}_{t=1}^T\in\Fcal^T} \Rcal_T^{\textsc{A}}$. If a lower bound for $\inf_{A\in\Acal} \sup_{\{f_t\}_{t=1}^T\in\Fcal^T} \Rcal_T^{\textsc{A}}$ matches the upper bound in~\eqref{eq:upperbound}, then we can say that POG is optimum for dynamic regret in online learning.



\begin{Theorem}
\label{theorem_lower_bound}
Assume that Assumptions \ref{assumption_convex_F_H} and \ref{assumption_bound_gradient_domain} hold. For any $0\le \beta < 1$, the lower bound for our problem with dynamic regret is 
\begin{align}
\nonumber
\inf_{A \in \mathcal{A}} \sup_{\{f_t\}_{t=1}^T\in\Fcal^T} \Rcal_T^{\textsc{A}} \gtrsim \sqrt{D_{\beta} \cdot T^{1-\beta}} + \sqrt{T},
\end{align}
where $\mathcal{A}$ is the set of all possible learning algorithms. $f_t(\x) = F_t(\x) + H(\x)$, $\forall t \in [T]$, with $\{f_t\}_{t=1}^T \in \Fcal^T$.
\end{Theorem}

The discussion for the lower bound is conducted in the following aspects.
\begin{itemize}
\item (\textbf{Insight.}) The lower bound in Theorem~\ref{theorem_lower_bound} can be interpreted by that for any algorithm there always exists a problem (or a function sequence in $\mathcal{F}^T$ such that the dynamic regret is not less than $\sqrt{D_{\beta} \cdot T^{1-\beta}} + \sqrt{T}$ up to a constant factor. It indicates that the lower bound matches with the upper bound shown in \eqref{eq:upperbound}. This theoretical result implies that the proximal online gradient is an optimal algorithm to find decisions in the dynamic environment defined by $D_{\beta}$ and our upper bound (shown in the following section) is also sufficiently tight. In addition, this lower bound also reveals the difficulty of following dynamics in online learning. 
\item (\textbf{Novelty.}) \cite{NeurIPS:2018:ZhangB} shows a lower bound for dynamic regret. Comparing with the known result, our lower bound has the following novelty.
\begin{itemize}
\item (General bound.) Our lower bound holds for any $0\le \beta <1$, but the result in \cite{NeurIPS:2018:ZhangB} only holds for the case of $\beta=0$. When $\beta>0$, it is the first work to show that the dynamic regret is $\Omega\lrincir{\sqrt{D_{\beta} \cdot T^{1-\beta}} + \sqrt{T}}$.
\item (Non-smooth $f_t$.) Our lower bound holds for the non-smooth sequence  $\{f_t\}_{t=1}^T$, but  \cite{NeurIPS:2018:ZhangB} only holds for the smooth sequence  $\{f_t\}_{t=1}^T$. 
\end{itemize}  
\end{itemize}

\subsection{Upper bound for a general dynamic regret ($0\le \beta < 1$)}
\label{subsection_upper_bound}
 We provide the upper bound for the POG algorithm described in Algorithm~\ref{algo_pog} in following. The complete proof is provided in the Appendix. It essentially follows the analysis framework for the online gradient algorithm. \textit{The main novelty lies that our analysis is more general than previous work. Our upper bound holds for a general dynamic regret, that is, $0\le \beta <1$, instead of $\beta =0$ in previous studies.}  
\begin{Theorem}
\label{theorem_our_upper_bound}
Let  $0\le \beta < 1$. Choose the positive learning rate sequence $\{\eta_t\}_{t=1}^T$ in Algorithm~\ref{algo_pog} to be non-increasing. Under Assumptions \ref{assumption_convex_F_H} and \ref{assumption_bound_gradient_domain}, the following upper bound for the dynamic regret holds 
\begin{align}
\label{equa_theorem_our_upper_bound}
\sup_{\{f_t\}_{t=1}^T\in\Fcal^T} \Rcal_T^{\textsc{POG}} \le & \sqrt{R} \max_{\{\eta_t\}_{t=1}^T} \left \{\frac{1}{\eta_t \cdot t^{\beta}} \right\} \cdot D_{\beta} + \frac{R}{2\eta_T} + \frac{G}{2} \sum_{t=1}^T \eta_t + H(\x_1)- H(\x_{T+1}).
\end{align}
\end{Theorem}

To the make the dynamic regret more clear, we choose the learning rate appropriately, which leads to the following result.
\begin{Corollary}
\label{corollary_unify_upper_bound_our}
For any $0\le \beta < 1$, we choose an appropriate $\gamma$  such that $\gamma \ge \beta$ and $0\le \gamma < 1$. Set the learning rate $\eta_t$ by
\begin{align}
\nonumber
\eta_t = t^{-\gamma} \cdot \sqrt{\frac{(1-\gamma)\lrincir{2\sqrt{R}T^{2\gamma - \beta-1}D_{\beta} + RT^{2\gamma-1}}}{G}}
\end{align} in Algorithm \ref{algo_pog}. Under Assumptions \ref{assumption_convex_F_H} and \ref{assumption_bound_gradient_domain}, we have  
\begin{align}
&\sup_{f_{t=1}^T\in\Fcal^T} \Rcal_T^{\textsc{POG}} 
\lesssim  \sqrt{D_{\beta} \cdot T^{1-\beta}} + \sqrt{T}.
\label{eq:upperbound}
\end{align}
\end{Corollary}


To compare the upper bound in \eqref{eq:upperbound} to existing results, we consider the special case which does not include the nonsmooth term $H(\cdot)$ in the objective and a particular choice for $\beta=0$. In such case, our upper bound is $\Ocal{\sqrt{TD_0} + \sqrt{T}}$, which is consistent with the known regret \cite{Gyorgy:2016,NeurIPS:2018:ZhangB}. Meanwhile, it slightly improves the known regret $\Ocal{\sqrt{T}  D_0 + \sqrt{T}}$ \cite{Zinkevich:2003,Hall:2013vr,Hall:2015ct} in the sense that it has a better dependence on $D_0$. When $\beta > 0$, our upper bound is $\Ocal{\sqrt{T^{1-\beta}D_{\beta}} + \sqrt{T}}$, which extends the known result for $\beta=0$ \cite{Gyorgy:2016,NeurIPS:2018:ZhangB}. Additionally, the upper bound in \cite{NeurIPS:2018:ZhangB} holds in the \textit{expert} setting, and requires smoothness of $f_t$. But, our upper bound holds in a general setting including the non-expert setting, and still holds for non-smooth $f_t$.

\textbf{Connections with $M-$shifting regret.} Although the shifting regret defined in \eqref{equa_definition_shift_regret} is different from the dynamic regret considered in this paper, it is worth noting that our result in \eqref{eq:upperbound} also implies an upper bound $\Ocal{\sqrt{MT} + \sqrt{T}}$ with respect to the shifting regret defined in \eqref{equa_definition_shift_regret}. 
\begin{Corollary}
\label{corollary_connect_shift_regret}
Set the learning rate $\eta_t$ by
\begin{align}
\nonumber
\eta_t = t^{-\gamma} \cdot \sqrt{\frac{(1-\gamma)\lrincir{2RT^{2\gamma -1}M + RT^{2\gamma-1}}}{G}}
\end{align} in Algorithm \ref{algo_pog}. Under Assumptions \ref{assumption_convex_F_H} and \ref{assumption_bound_gradient_domain}, we have    
\begin{align}
\nonumber
\sup_{f_{t=1}^T\in\Fcal^T} \widetilde{\Rcal}_T^{\textsc{POG}} \lesssim  \sqrt{M T} + \sqrt{T},
\end{align}
\end{Corollary}
where $ \widetilde{\Rcal}_T^{\textsc{POG}}$ follows the definition in \eqref{equa_definition_shift_regret}. This result slightly improves the existing result for shifting regret in \cite{pmlr-v54-jun17a} up to a logarithmic factor. The proof is provided in the appendix.

\section{Conclusion}
\label{sect_conclusion}
The online learning problem with dynamic regret metric is particularly interesting for many real sceneiros. Although the online gradient method
has been shown to be optimal for the static
regret metric, the optimal algorithm for the
dynamic regret remains unknown. This paper studies this problem from a theoretical prespective. We show that proximal online gradient, a general version of online gradient, is optimum to the dynamic regret by showing
that our proved lower bound matches the
upper bound which slightly improves the existing
upper bound.

\section*{Appendix: Proofs}
In this section, we present the detailed proofs for the theorems in our paper. In particular, \textit{ Some necessary lemmas used in proofs to theorems are placed in supplementary materials. }

In our proofs, we abuse the notations of $\partial H(\x)$ a little bit to represent any vector in the subgradient of $H(\x)$. $G_t(\x)$ still represents any vector in $\partial F_t(\x)$. We use $B_{\psi}(\x,\y): = \psi(\x)-\psi(\y) - \langle \psi(\y), \x-\y \rangle$ to denote Bregman divergence w.r.t. the function $\psi$.

\begin{Lemma}
\label{lemma_rademacher_expectation} 
Consider a sequence $\{\v_t\}_{t=1}^{T}$. For any $t \in [T]$, dimensions of $\v_t \in \{\pm 1\}^d$ are i.i.d. sampled from  Rademacher distribution. We have
\begin{align*}
\EE_{\{\v_t\}_{t=1}^T} \lrnorm{\sum_{t=1}^{T} \v_t }_1 \gtrsim d \sqrt{T} 
\end{align*} 
\end{Lemma} 

\begin{proof}
We consider the left hand side
\begin{align}
\label{equa_random_walk_sqrt_t}
\EE_{\{\v_t\}_{t=1}^T} \lrnorm{\sum_{t=1}^{T} \v_t }_1 = \EE_{\{\v_t\}_{t=1}^T} \sum_{i = 1}^d \bigg | \sum_{t=1}^{T} \v_t(i) \bigg | = d \cdot \EE_{\{\v_t(1)\}_{t=1}^T} \bigg | \sum_{t=1}^{T} \v_t(1) \bigg |, 
\end{align}
where $\v_t(i)$ denotes the $i$-th dimension of $\v_t$, and $\{\v_t(1)\}_{t=1}^T := \{\v_1(1), \v_2(1), ..., \v_T(1)\}$.
The second equality holds because that every dimension of $\v_t$ is independent to each other.

Consider the sequence $\{\v_{t}\}_{t=1}^T$. If the event: \textit{$+1$ is picked} happens $m$ times with the probability $P_m$, then the event : \textit{$-1$ is picked} happens $T-m$ times. Denote $S_T := \sum_{t=1}^{T} \v_t(1)$, and we have  
\begin{align}
\nonumber
S_T = m - (T-m)  = 2m-T.
\end{align} Denote $\Scal:=\{-T, -T+2, ..., T-2, T\}$, and $S_T \in \Scal$. Thus, we have
\begin{align}
\nonumber
P(S_T = 2m-T) = P_m = \frac{1}{2^T}\cdot \binom{T}{m},
\end{align} 
and
\begin{align}
\nonumber
\EE |S_T| = & \sum_{m=0}^T  \frac{\lvert 2m-T \rvert}{2^T} \cdot \binom{T}{m} = \frac{1}{2^T}\cdot \sum_{m=0}^T  \frac{\lvert 2m-T\rvert \cdot T!}{m! \cdot \lrincir{T-m}!}.
\end{align} When $T$ is even, denote $T = 2J$. Thus,
\begin{align}
\nonumber
& \EE |S_T| \\ \nonumber 
= & \frac{1}{2^{2J}}\cdot \sum_{m=0}^T \frac{\lvert 2m-2J\rvert \cdot (2J)!}{m! \cdot (2J-m)!} \\ \nonumber
= & \frac{(2J)!}{2^{2J}}\cdot \sum_{m=0}^{2J}  \frac{\lvert 2m-2J\rvert }{m! \cdot (2J-m)!} \\ \nonumber
\refabovecir{=}{\textcircled{1}} & \frac{(2J)!}{2^{2J-2}}\cdot \sum_{n=1}^J  \frac{n}{(J+n)! \cdot (J-n)!} \\ \nonumber
=& \frac{1}{2^{2J-2}}\cdot   \lrincir{ \sum_{n=0}^J (n+J)\binom{2J}{J+n} -  \sum_{n=0}^J J\binom{2J}{J+n} } \\ \nonumber 
=& \frac{1}{2^{2J-2}}\cdot   \lrincir{ \sum_{i=J}^{2J} i \binom{2J}{i} -  \sum_{i=J}^{2J} J\binom{2J}{i} } \\ \nonumber 
\refabovecir{=}{\textcircled{2}}& \frac{1}{2^{2J-2}}\cdot   \lrincir{ \sum_{i=J}^{2J} 2J \binom{2J-1}{i-1} -  \sum_{i=J}^{2J} J\binom{2J}{i} } \\ \nonumber 
\refabovecir{=}{\textcircled{3}}& \frac{2J}{2^{2J-2}}\cdot   \lrincir{ \sum_{k=J-1}^{2J-1} \binom{2J-1}{k} -  \frac{1}{4}\lrincir{2^{2J} + \binom{2J}{J}} } \\ \nonumber 
\refabovecir{=}{\textcircled{4}}& \frac{2J}{2^{2J-2}}\cdot   \lrincir{ \frac{1}{2}\lrincir{ 2^{2J-1} + \binom{2J-1}{J-1} } -  \frac{1}{4}\lrincir{2^{2J} + \binom{2J}{J}} } \\ \nonumber 
=& \frac{2J}{2^{2J-2}}\cdot   \lrincir{ \frac{1}{4} \binom{2J}{J} } \\ \nonumber
= & \frac{2J}{ 4^J}\cdot \frac{(2J)!}{ J! \cdot J!} \\ \nonumber
\refabovecir{\ge}{\textcircled{5}} & 2J \cdot \frac{1}{2\sqrt{J}} \\ \nonumber
= & \sqrt{\frac{T}{2}}.
\end{align}   Here, $\textcircled{1}$ holds due to
\begin{align}
\nonumber
& \frac{(2J)!}{2^{2J}}\cdot \sum_{m=0}^{2J}  \frac{\lvert 2m-2J\rvert }{m! \cdot (2J-m)!} \\ \nonumber
=&  \frac{(2J)!}{2^{2J}}\cdot \lrincir{\sum_{m=0}^J  \frac{2J-2m}{m! \cdot (2J-m)!} +  \sum_{m=J+1}^{2J}  \frac{2m-2J}{m! \cdot (2J-m)!}} \\ \nonumber
= & \frac{(2J)!}{2^{2J}}\cdot \sum_{n_1=0}^J  \frac{2n_1}{(J-n_1)! \cdot (J+n_1)!} +  \frac{(2J)!}{2^{2J}} \cdot \sum_{n_2=1}^{J}  \frac{2n_2}{(J+n_2)! \cdot (J-n_2)!} \\ \nonumber
= & \frac{(2J)!}{2^{2J}}\cdot \sum_{n_1=0}^J  \frac{2n_1}{(J-n_1)! \cdot (J+n_1)!} +  \frac{(2J)!}{2^{2J}} \cdot \sum_{n_2=0}^{J}  \frac{2n_2}{(J+n_2)! \cdot (J-n_2)!} \\ \nonumber
= & \frac{(2J)!}{2^{2J-2}}\cdot \sum_{n=0}^J  \frac{n}{(J-n)! \cdot (J+n)!}.
\end{align} $\textcircled{2} $ holds because that, for any $0\le k\le N$,
\begin{align}
\nonumber
k\binom{N}{k} = N\binom{N-1}{k-1}
\end{align} $\textcircled{3} $  holds because that
\begin{align}
\nonumber
2^{2J}  = & \sum_{i=0}^J \binom{2J}{i} + \sum_{i=J}^{2J} \binom{2J}{i} - \binom{2J}{J} = 2\sum_{i=0}^J \binom{2J}{i} - \binom{2J}{J}.
\end{align} $\textcircled{4} $  holds because that
\begin{align}
\nonumber
2^{2J-1}  = & \sum_{i=0}^{J-1} \binom{2J-1}{i} + \sum_{i=J-1}^{2J-1} \binom{2J-1}{i} - \binom{2J-1}{J-1} = 2\sum_{i=0}^J \binom{2J-1}{i} - \binom{2J-1}{J-1}.
\end{align} $\textcircled{5} $ holds because that, for any $n > 1$,
\begin{align}
\nonumber
\frac{1}{4^n}\binom{2n}{n} \ge \frac{1}{2\sqrt{n}}.
\end{align} 

When $T$ is odd, we have 
\begin{align}
\nonumber
\EE |S_T| = & \EE |S_{T-1}+\v_{T}(1)| \\ \nonumber 
\ge & \EE |S_{T-1}| - \EE |\v_{T}(1)| \\ \nonumber 
= & \EE |S_{T-1}| - 1 \\ \nonumber 
= &\sqrt{\frac{T}{2}}-1.
\end{align} Finally, we obtain
\begin{align}
\nonumber
\EE_{\{\v_t\}_{t=1}^T} \lrnorm{\sum_{t=1}^{T} \v_t }_1 \gtrsim d\sqrt{T}.
\end{align}
It completes the proof.
\end{proof}

\balance

\noindent \textbf{Proof to Theorem \ref{theorem_lower_bound}:}
\begin{proof}
Let $f_t(\x_t) = F_t(\x_t) + H(\x_t)$, where $ F_t(\x_t) := \langle \v_t, \x_t \rangle$ and $H(\x_t) = 0$ for all $\x_t \in \Xcal$. Here, $\v_t \in \{+1,-1\}^d$ is a random vector with i.i.d. elements sampled from Rademacher distribution. $\mathcal{X} = \left\{ \x \in \mathbb{R}^d : \|\x \|_2 \leq 1 \right \}$, and $\mathcal{L}_{D_{\beta}}^T = \{\{\y_t\}_{t=1}^T  : \sum_{t=1}^{T-1} t^{\beta}\cdot \|\y_{t+1}-\y_t\|_2 \leq D_{\beta} \}$. Under this construction, for any given algorithm $A\in \mathcal{A}$, we have

\begin{align}
& \sup_{\{f_t\}_{t=1}^T} \Rcal_T^A \\ \nonumber 
=& \sup_{\{\v_t\}_{t=1}^T} \Rcal_T^A \geq  \EE_{\{\v_t\}_{t=1}^T} \Rcal_T^A
\\ \nonumber
=&\EE_{\{\v_t\}_{t=1}^T} \sum_{t=1}^T f_t(\x_t) - \EE_{\{\v_t\}_{t=1}^T} \lrincir{\min_{\{\y_t\}_{t=1}^T \in \mathcal{L}_{D_{\beta}}^T} \sum_{t=1}^T f_t(\y_t) }
\\ \nonumber
=&\EE_{\{\v_t\}_{t=1}^T} \sum_{t=1}^T \langle \v_t, \x_t \rangle - \EE_{\{\v_t\}_{t=1}^T} \lrincir{ \min_{\{\y_t\}_{t=1}^T \in \mathcal{L}_{D_{\beta}}^T} \sum_{t=1}^T \langle \v_t, \y_t \rangle }
\\ \nonumber
=& 0 - \EE_{\{\v_t\}_{t=1}^T} \lrincir{ \min_{\{\y_t\}_{t=1}^T \in \mathcal{L}_{D_{\beta}}^T} \sum_{t=1}^T \langle \v_t, \y_t \rangle }
\\ \nonumber
=& \EE_{\{\v_t\}_{t=1}^T} \lrincir{ \max_{\{\y_t\}_{t=1}^T \in \mathcal{L}_{D_{\beta}}^T} \sum_{t=1}^T \langle -\v_t, \y_t \rangle }
\\ \nonumber
\refabovecir{=}{\textcircled{1} }& \EE_{\{\v_t\}_{t=1}^T} \lrincir{ \max_{\{\y_t\}_{t=1}^T \in \mathcal{L}_{D_{\beta}}^T} \sum_{t=1}^T \langle \v_t, \y_t \rangle}
\end{align}
$\textcircled{1}$ holds since Rademacher distribution is a symmetric distribution.

Next, we try to estimate the lower bound of $\EE_{\{\v_t\}_{t=1}^T} \max_{\{\y_t\}_{t=1}^T \in \mathcal{L}_{D_{\beta}}^T} \sum_{t=1}^T \langle \v_t, \y_t \rangle$. 

One feasible solution of $\y_t$ is constructed as follows,
\begin{enumerate}
    \item Evenly split the sequence $\{\y_t\}_{t=1}^T$ into two sub-sequences: $\{\bar{\y}_t\}_{t=1}^{T_1} := \{\y_t\}_{t=1}^{\frac{T}{2}-1}$ and $\{\hat{\y}_t\}_{t=1}^{T_2} := \{\y_t\}_{t=\frac{T}{2}}^{T}$, where $T_1 = T_2 = \frac{T}{2}$. $\{\v_t\}_{t=1}^T$ is also split into $\{\bar{\v}_t\}_{t=1}^{T_1}$ and $\{\hat{\v}_t\}_{t=1}^{T_2}$.
    \item Let all the $\lrnorm{\y_t}_2 \leq \frac{1}{2}$,
    \item Evenly split $\{\bar{\y}_t\}_{t=1}^{T_1}$ into $N := \left \lceil \frac{D_{\beta}}{T_1^\beta} \right \rceil $ subsets $\{\y_t\}_{t=1}^{\frac{T_1}{N}}$, $\{\y_t\}_{t=\frac{T_1}{N} + 1}^{\frac{2T_1}{N}}$, $\{\y_t\}_{t=\frac{2T_1}{N} + 1}^{\frac{3T_1}{N}}$, ..., $\{\y_t\}_{t=\frac{(N-1)T_1}{N}+1}^{T_1}$.
    \item For the first sub-sequence $\{\bar{\y}_t\}_{t=1}^{T_1}$, within $i$-th subset, let the values in it be same, and denote it by $\u_i$. For the second sub-sequence $\{\hat{\y}_t\}_{t=1}^{T_2}$, let all values be $\u_N$. 
    \item Since elements in the second sub-sequence $\{\hat{\y}_t\}_{t=1}^{T_2}$ have the same value $\u$, the difference between two elements is $0$. Additionally, consider the first sub-sequence $\{\bar{\y}_t\}_{t=1}^{T_1}$. Elements in different subsets can be different such that $\|\u_{i+1} - \u_i\| \leq \|\u_{i+1}\| + \|\u_i\| \leq 1$. We have 
    \begin{align}
    \nonumber
    & \sum_{t=1}^{T-1} t^{\beta}\cdot\|\y_{t+1} - \y_{t}\| = \sum_{t=1}^{T_1-1} t^{\beta}\cdot\|\y_{t+1} - \y_{t}\| + 0 \\ \nonumber 
    =& \sum_{i=1}^{N-1} \|\u_{i+1} - \u_i\| \cdot \lrincir{\frac{T_1}{N}\cdot i}^\beta  \\ \nonumber
    \le & T_1^{\beta}\sum_{i=1}^{N-1} \lrincir{\frac{i}{N}}^{\beta} \\ \nonumber
    \le & T_1^{\beta}(N-1) \\ \nonumber
    \le & D_{\beta}.
    \end{align} It implies $\{\bar{\y}_t\}_{t=1}^{T_1}$ and $\{\hat{\y}_t\}_{t=1}^{T_2}$ under our construction are feasible.
\end{enumerate}
Based on the above steps, we have
\begin{align}
&\EE_{\{\v_t\}_{t=1}^T} \lrincir{ \max_{\{\y_t\}_{t=1}^T \in \mathcal{L}_{D_{\beta}}^T} \sum_{t=1}^T \langle \v_t, \y_t \rangle} 
\\ \nonumber
= &\EE_{\{\bar{\v}_t\}_{t=1}^{T_1}} \max_{\{\u_i\}_{i=1}^N} \sum_{i=1}^N \lrangle{\sum_{t=\frac{iT_1}{N}+1}^{\frac{(i+1)T_1}{N}}\bar{\v}_t, \u_i } + \EE_{\{\hat{\v}_t\}_{t=1}^{T_2}} \max_{\u}  \sum_{t=1}^{T_2} \langle \hat{\v}_t, \u \rangle \\ \nonumber 
= &\frac{1}{2}\EE_{\{\bar{\v}_t\}_{t=1}^{T_1}} \max_{\{\z_i\}_{i=1}^N \in \Xcal^N} \sum_{i=1}^N \lrangle{\sum_{t=\frac{iT_1}{N}+1}^{\frac{(i+1)T_1}{N}}\bar{\v}_t, \z_i } + \frac{1}{2}\EE_{\{\hat{\v}_t\}_{t=1}^{T_2}} \max_{\z \in \Xcal}  \sum_{t=1}^{T_2} \langle \hat{\v}_t, \z \rangle 
\\ \nonumber
\refabovecir{=}{\textcircled{1}} & \frac{1}{2} N \cdot \EE_{\{\bar{\v}_t\}_{t=1}^{T_1}}  \lrnorm{ \sum_{t=\frac{iT_1}{N}+1}^{\frac{(i+1)T_1}{N}}\bar{\v}_t} + \frac{1}{2}\EE_{\{\hat{\v}_t\}_{t=1}^{T_2}} \lrnorm{\sum_{t=1}^{T_2}\hat{\v}_t}
\\ \nonumber
\refabovecir{\geq}{\textcircled{2}}  &   \frac{1}{2} N \frac{1}{\sqrt{d}}   \EE_{\{\bar{\v}_t\}_{t=1}^{T_1}} \lrnorm{ \sum_{t=\frac{iT_1}{N}+1}^{\frac{(i+1)T_1}{N}}\v_t}_1 + \frac{1}{2\sqrt{d}} \EE_{\{\hat{\v}_t\}_{t=1}^{T_2}} \lrnorm{\sum_{t=1}^{T_2}\hat{\v}_t}_1 
\\ \nonumber
\refabovecir{\gtrsim }{\textcircled{3}}  & \frac{\sqrt{d}}{2} \cdot \sqrt{T_1N} + \frac{\sqrt{d}}{2} \cdot \sqrt{T_2} 
\\  \label{equa_lower_bound_omega}
\gtrsim & \sqrt{D_{\beta}\cdot T^{1-\beta}} + \sqrt{T}.
\end{align} Recall $\mathcal{X} = \{\x \in \mathbb{R}^d : \|\x \|_2 \leq 1 \}$ in this example. $\textcircled{1}$ holds due to the definition of the dual norm of $\ell_2$ norm, specifically, which is $\lrnorm{\x}_\ast = \lrnorm{\x}_2 = \max_{\lrnorm{\y}\le 1} \lrangle{\x, \y}.$ since the dual norm of $\ell_2$ norm is still $\ell_2$ norm. $\textcircled{2}$ holds due to $\|\x\|_1 \leq \sqrt{d} \|\x\|$. $\textcircled{3}$ holds due to Lemma~\ref{lemma_rademacher_expectation}.

Since \eqref{equa_lower_bound_omega} holds for any algorithm $A \in \Acal$, we thus obtain
\[
    \inf_{A \in \mathcal{A}} \sup_{\{f_t\}_{t=1}^T \in \Fcal^T} \Rcal_T^A = \Omega\lrincir{\sqrt{D_{\beta} \cdot T^{1-\beta}} + \sqrt{T}}.
\]
It completes the proof.
\end{proof}

\noindent \textbf{Proof to Theorem \ref{theorem_our_upper_bound}:}
\begin{proof}
For any sequence of $T$ loss functions $\{f_t\}_{t=1}^T \in \Fcal^T$, we have 
\begin{align}
\nonumber
& \sum_{t=1}^T \lrincir{F_t(\x_t) + H(\x_t) - F_t(\y_t) - H(\y_t) } \\ \nonumber 
= & \underbrace{\sum_{t=1}^T \lrincir{F_t(\x_t) + H(\x_{t+1}) - F_t(\y_t) - H(\y_t)}}_{I_0}  + H(\x_{1}) - H(\x_{T+1}).
\end{align} 

According to Lemma \ref{lemma_1}, we have
\begin{align}
\nonumber
I_0 = & \sum_{t=1}^T \lrincir{F_t(\x_t) + H(\x_{t+1}) - F_t(\y_t) - H(\y_t) } \\ \nonumber 
\le & \sum_{t=1}^T \frac{1}{2\eta_t} \lrincir{ \lrnorm{\y_t- \x_t}_2^2 - \lrnorm{\y_t - \x_{t+1}}_2^2} + \frac{1}{2} \sum_{t=1}^T \eta_t \lrnorm{G_t(\x_t)}^2 \\ \nonumber
\refabovecir{\le}{\textcircled{1}} & \sqrt{R}  \sum_{t=1}^{T-1}\frac{1}{\eta_t}\lrincir{\lrnorm{\y_{t+1} - \y_t}}  + \frac{R}{2\eta_T} + \frac{G}{2} \sum_{t=1}^T \eta_t \\ \nonumber
\le & \sqrt{R} \max_{\eta_{t=1}^T} \left \{\frac{1}{\eta_t \cdot t^{\beta}} \right\} \cdot D_{\beta}  + \frac{R}{2\eta_T} + \frac{G}{2} \sum_{t=1}^T \eta_t.
\end{align} $\textcircled{1}$ holds due to Lemma \ref{lemma_recurrsive_bound}. Thus, we have
\begin{align}
\label{equa_theorem1_R_temp1} 
 \sum_{t=1}^T \lrincir{F_t(\x_t) + H(\x_t) - F_t(\y_t) - H(\y_t) } \le \sqrt{R} \max_{\eta_{t=1}^T} \left \{\frac{1}{\eta_t \cdot t^{\beta}} \right\} \cdot D_{\beta}  + \frac{R}{2\eta_T} + \frac{G}{2} \sum_{t=1}^T \eta_t + H(\x_1)- H(\x_{T+1}).
\end{align} Since \eqref{equa_theorem1_R_temp1} holds for any sequence of loss functions $\{f_t\}_{t=1}^T \in \Fcal^T$, thus, 
\begin{align}
\nonumber
\sup_{\{f_t\}_{t=1}^T \in \Fcal^T}\Rcal_T^{\textsc{POG}} \le & \sqrt{R} \max_{\{\eta_t\}_{t=1}^T} \left \{\frac{1}{\eta_t \cdot t^{\beta}} \right\} \cdot D_{\beta}  + \frac{R}{2\eta_T} + \frac{G}{2} \sum_{t=1}^T \eta_t + H(\x_1) - H(\x_{T+1})
\end{align}
It completes the proof.
\end{proof}

\noindent \textbf{Proof to Corollary \ref{corollary_unify_upper_bound_our}}
\begin{proof}
Assume $\eta_t := t^{-\gamma}\cdot \sigma_1$, where $\sigma_1$ is a constant, and does not depend on $t$. According to Theorem \ref{theorem_our_upper_bound}, when $\gamma \ge \beta$, 
\begin{align}
\nonumber
\max_{\eta_{t=1}^T} \left \{\frac{1}{\eta_t \cdot t^{\beta}} \right\} = \frac{T^{\gamma - \beta}}{\sigma_1}.
\end{align} Substituting it into \eqref{equa_theorem_our_upper_bound}, we have
\begin{align}
\nonumber
& \Rcal_T^{\textsc{POG}} \\ \nonumber 
\le & \frac{\sqrt{R} D_{\beta}}{\sigma_1} T^{\gamma -\beta}  + \frac{R}{2\sigma_1}T^{\gamma} + \frac{G\sigma_1}{2} \sum_{t=1}^{T} t^{-\gamma} + H(\x_1) - H(\x_{T+1}) \\ \nonumber
\refabovecir{\le}{\textcircled{1}} & \frac{\sqrt{R} D_{\beta}}{\sigma_1} T^{\gamma -\beta}  + \frac{R}{2\sigma_1}T^{\gamma} + \frac{G\sigma_1}{2(1-\gamma)}T^{1-\gamma} + H(\x_1) - H(\x_{T+1}).
\end{align} $\textcircled{1}$ holds due to $0\le \gamma < 1$, and Lemma \ref{lemma_seriez_upper_bound_temp}.

Choosing the optimal $\sigma_1$ with 
\begin{align}
\nonumber
\sigma_1 = \sqrt{\frac{(1-\gamma)\lrincir{2\sqrt{R}T^{2\gamma - \beta-1}D_{\beta} + RT^{2\gamma-1}}}{G}},
\end{align} we have
\begin{align}
\nonumber
\Rcal_T^{\textsc{POG}} \le & \sqrt{\frac{2G\sqrt{R}D_{\beta} T^{1-\beta}}{1-\gamma}} + \sqrt{\frac{GRT}{4(1-\gamma)}} + H(\x_1) - H(\x_{T+1}) \\ \nonumber 
\lesssim & \sqrt{D_{\beta} \cdot T^{1-\beta}} + \sqrt{T}.
\end{align} 
It completes the proof.

\end{proof}

\begin{Lemma}
\label{lemma_connect_dynamic_regret_shifting_regret_equivalent}
The optimal reference points $\{\y_t\}_{t=1}^T$  satisfying $\sum_{t=1}^{T-1} \mathbbm{1}\{\y_{t+1} \neq \y_t\} \le  M$ still satisfy $\sum_{t=1}^{T-1}\lrnorm{\y_{t+1} - \y_t} \le M \sqrt{R}$.
\end{Lemma}
\begin{proof}
Denote $a_t = \lrnorm{\y_{t+1} - \y_t}$, and $\a_{T} = \left \{\a_t | t \in [T-1] \right \} \in \RR^{T-1}$.  Note that $\sum_{t=1}^{T-1} \mathbbm{1}\{\y_{t+1} \neq \y_t\} = \lrnorm{\a_T}_0$. Thus, for $M $-shifting regret, $\sum_{t=1}^{T-1} \mathbbm{1}\{\y_{t+1} \neq \y_t\} = \lrnorm{\a_T}_0 \le  M$. When $\beta = 0$, we have 
\begin{align}
\nonumber
\sum_{t=1}^{T-1}\lrnorm{\y_{t+1} - \y_t} = \lrnorm{\a_T}_1 \le \lrnorm{\a_T}_0 \sqrt{R} \le M \sqrt{R}.
\end{align} The first inequality holds because, for any $1\le t \le T$, $\lrnorm{\y_{t+1} - \y_t} \le \sqrt{R}$. It completes the proof.
\end{proof}

\textbf{Proof to Corollary \ref{corollary_connect_shift_regret}:}

\begin{proof}
Replacing $D_{0}$ by $M \sqrt{R}$ in Corollary \ref{corollary_unify_upper_bound_our}, we have 
\begin{align}
\nonumber
\sup_{\{f_t\}_{t=1}^T\in\Fcal^T} \Rcal_T^{\textsc{POG}} \lesssim \sqrt{MT} + \sqrt{T}.
\end{align} According to Lemma \ref{lemma_connect_dynamic_regret_shifting_regret_equivalent}, we obtain $\sup_{f_{t=1}^T\in\Fcal^T} \widetilde{\Rcal}_T^{\textsc{POG}} \le \sup_{\{f_t\}_{t=1}^T\in\Fcal^T} \Rcal_T^{\textsc{POG}} \lesssim \sqrt{MT} + \sqrt{T}$. It thus completes the proof.
\end{proof}

\begin{Lemma}
\label{lemma_1}
Given any sequence $\{\y_t\}_{t=1}^T \in \mathcal{L}_{D_{\beta}}^T$, and setting any $\eta_t>0$ in Algorithm \ref{algo_pog}, we have 
\begin{align}
\nonumber
& \sum_{t=1}^T \lrincir{F_t(\x_t) + H(\x_{t+1}) - F_t(\y_t) - H(\y_t) } \\ \nonumber 
\le & \sum_{t=1}^T \frac{1}{2\eta_t} \lrincir{ \lrnorm{\y_t - \x_t}_2^2 - \lrnorm{\y_t - \x_{t+1}}_2^2} +  \sum_{t=1}^T \frac{\eta_t}{2} \lrnorm{G_t(\x_t)}^2.
\end{align}
\end{Lemma} 
\begin{proof}
Define  $\psi(\x) := \frac{1}{2}\lrnorm{\x}^2$, and $\x_{t+1} = \argmin_{\x\in\Xcal} \lrangle{G_t(\x_t), \x} + \frac{1}{\eta_t} B_\psi(\x,\x_t) + H(\x)$, according to the optimal condition, for any $\x\in\Xcal$, we have 
\begin{align}
\label{equa_optimal_condition_composite}
0\le & \lrangle{ \x - \x_{t+1}, \eta_t G_t(\x_t)} + \lrangle{\x - \x_{t+1}, \nabla \psi(\x_{t+1}) - \nabla \psi(\x_t) + \eta_t \partial H(\x_{t+1}) } .
\end{align} Then, we have
\begin{align}
\nonumber
&\eta_t\lrincir{F_t(\x_t) + H(\x_{t+1}) - F_t(\y_t) - H(\y_t)} \\ \nonumber
\le & \eta_t \lrangle{\x_t - \y_t, G_t(\x_t)} + \eta_t \lrangle{\x_{t+1}-\y_t, \partial H(\x_{t+1})} \\ \nonumber
= & \eta_t \lrangle{\x_{t+1} - \y_t, G_t(\x_t)} + \eta_t \lrangle{\x_{t+1}-\y_t, \partial H(\x_{t+1})}  + \eta_t \lrangle{\x_t - \x_{t+1}, G_t(\x_t)} \\ \nonumber
\refabovecir{\le}{\textcircled{1}} & \lrangle{\y_t - \x_{t+1}, \nabla \psi(\x_{t+1}) - \nabla \psi(\x_t)} + \eta_t \lrangle{\x_t - \x_{t+1}, G_t(\x_t)} \\ \nonumber
\refabovecir{=}{\textcircled{2}} & B_{\psi}(\y_t, \x_t) - B_{\psi}(\x_{t+1}, \x_t) - B_{\psi}(\y_t, \x_{t+1}) + \eta_t \lrangle{\x_t - \x_{t+1}, G_t(\x_t)} \\ \nonumber
\refabovecir{\le}{\textcircled{3}} & B_{\psi}(\y_t, \x_t) - B_{\psi}(\y_t, \x_{t+1}) + \frac{\eta_t^2}{2}\lrnorm{G_t(\x_t)}^2.
\end{align} $\textcircled{1}$ holds due to \eqref{equa_optimal_condition_composite}. $\textcircled{2}$ holds due to three-point identity for Bregman divergence, which is, for any vectors $\x$, $\y$, and $\z$, 
\small\begin{align*}
B_{\psi}(\x, \y) =  B_{\psi}(\x, \z) + B_{\psi}(\z, \y) -\lrangle{\x-\z, \nabla \psi(\y) - \nabla \psi(\z)}.
\end{align*} \normalsize
$\textcircled{3}$ holds due to $\psi(\x) = \frac{1}{2}\lrnorm{\x}_2^2$, so that $B_{\psi}(\x_{t+1}, \x_t) = \frac{1}{2}\lrnorm{\x_{t+1} - \x_t}_2^2$. Thus, we finally obtain
\begin{align}
\nonumber
& \sum_{t=1}^T \lrincir{F_t(\x_t) + H(\x_{t+1}) - F_t(\y_t) - H(\y_t) } \\ \nonumber 
\le & \sum_{t=1}^T \frac{B_{\psi}(\y_t, \x_t) - B_{\psi}(\y_t, \x_{t+1}) }{\eta_t}  + \frac{1}{2} \sum_{t=1}^T \eta_t \lrnorm{G_t(\x_t)}^2 \\ \nonumber
= & \sum_{t=1}^T \frac{\lrnorm{\y_t - \x_t}_2^2 - \lrnorm{\y_t - \x_{t+1}}_2^2}{2\eta_t}  +  \sum_{t=1}^T \frac{\eta_t}{2} \lrnorm{G_t(\x_t)}^2.
\end{align} It completes the proof.
\end{proof}

\begin{Lemma}
\label{lemma_recurrsive_bound} 

Given any sequence $\{\y_t\}_{t=1}^T \in \mathcal{L}_{D_{\beta}}^T$, and setting a non-increasing series $0<\eta_{t+1} \le \eta_t$ in Algorithm \ref{algo_pog}, we have  
\begin{align}
\nonumber
\sum\limits_{t=1}^{T} \frac{1}{\eta_t} \lrincir{  -  \lrnorm{\y_t - \x_{t+1}}^2   + \lrnorm { \y_t- \x_t }^2 } \le 2\sqrt{R}  \sum_{t=1}^{T-1}\frac{1}{\eta_t}\lrincir{\lrnorm{\y_{t+1} - \y_t}}  + \frac{R}{\eta_T}.
\end{align} 
\end{Lemma} 
\begin{proof}

According to the law of cosines, we have 
\begin{align}
\nonumber
& -  \lrnorm{\y_t - \x_{t+1}}^2   + \lrnorm { \y_{t+1} - \x_{t+1} }^2 \\ \nonumber 
\le & 2 \lrnorm {\y_{t+1} - \y_t} \lrnorm{\x_{t+1} - \y_{t+1}} - \lrnorm{\y_{t+1} - \y_t}^2 \\ \nonumber
\le & 2\sqrt{R} \lrnorm {\y_{t+1} - \y_t}  - \lrnorm{\y_{t+1} - \y_t}^2 \\ \label{equa_cosine}
\le & 2\sqrt{R} \lrnorm {\y_{t+1} - \y_t}.
\end{align}

Thus, we obtain
\begin{align}
\nonumber
& \sum\limits_{t=1}^{T} \frac{1}{\eta_t} \lrincir{  -  \lrnorm{\y_t - \x_{t+1}}^2   + \lrnorm { \y_t- \x_t }^2 } \\ \nonumber
=&\sum\limits_{t=1}^{T-1}  \lrincir{  -  \frac{1}{\eta_t}\lrnorm{\y_t - \x_{t+1}}^2   + \frac{1}{\eta_{t+1}}\lrnorm { \y_{t+1} - \x_{t+1} }^2} + \frac{1}{\eta_1}\lrnorm{\y_1 - \x_1}^2 - \frac{1}{\eta_T} \lrnorm{\y_T - \x_{T+1}}^2 \\ \nonumber
\le &\sum\limits_{t=1}^{T-1}  \lrincir{  -  \frac{1}{\eta_t}\lrnorm{\y_t - \x_{t+1}}^2   + \frac{1}{\eta_{t}}\lrnorm { \y_{t+1} - \x_{t+1} }^2} + \sum\limits_{t=1}^{T-1} \lrincir{\frac{1}{\eta_{t+1}} - \frac{1}{\eta_t}} \lrnorm { \y_{t+1} - \x_{t+1} }^2  + \frac{1}{\eta_1}\lrnorm{\y_1 - \x_1}^2 \\ \nonumber
\le &\sum\limits_{t=1}^{T-1}  \lrincir{  -  \frac{1}{\eta_t}\lrnorm{\y_t - \x_{t+1}}^2   + \frac{1}{\eta_{t}}\lrnorm { \y_{t+1} - \x_{t+1} }^2} + R\sum\limits_{t=1}^{T-1} \lrincir{\frac{1}{\eta_{t+1}} - \frac{1}{\eta_t}} +\frac{R}{\eta_1} \\ \nonumber
\refabovecir{\le}{\textcircled{1}} & 2\sqrt{R}  \sum_{t=1}^{T-1}\frac{1}{\eta_t}\lrincir{\lrnorm{\y_{t+1} - \y_t}}  + \frac{R}{\eta_T}.
\end{align} $\textcircled{1}$ holds due to \eqref{equa_cosine}.
The proof is completed.
\end{proof}

\begin{Lemma}
\label{lemma_seriez_upper_bound_temp}
For any $0\le \gamma <1$, we have
\begin{align}
\nonumber
\sum_{t=1}^{T} \frac{1}{t^{\gamma}} \le \frac{1}{1-\gamma} T^{1-\gamma}.
\end{align}

\end{Lemma}
\begin{proof}
We will use mathematical induction method to prove the result.
Given $0\le \gamma < 1$, it is trivial to verify that 
\begin{align}
\nonumber
\frac{1}{1^{\gamma}} = 1 \le \frac{1}{1-\gamma}.
\end{align} For an integer $T_0$, suppose $\sum_{t=1}^{T_0} \frac{1}{t^{\gamma}} \le \frac{1}{1-\gamma} T_0^{1-\gamma}$. Then, we have
\begin{align}
\nonumber
& \sum_{t=1}^{T_0+1} \frac{1}{t^{\gamma}} = \sum_{t=1}^{T_0} \frac{1}{t^{\gamma}} + \frac{1}{(T_0+1)^{\gamma}} \\ \nonumber
\le & \frac{1}{1-\gamma} T_0^{1-\gamma} + \frac{1}{(T_0+1)^{\gamma}} \\ \nonumber
= & \frac{1}{1-\gamma} (T_0+1)^{1-\gamma} \lrincir{ \lrincir{\frac{T_0}{T_0+1}}^{1-\gamma} + \frac{1-\gamma}{T_0 + 1} } \\ \nonumber
\refabovecir{\le}{\textcircled{1}} & \frac{1}{1-\gamma} (T_0+1)^{1-\gamma} \lrincir{ 1 - \frac{1-\gamma}{T_0+1} - \frac{\gamma (1-\gamma)}{2(T_0+1)^2} + \frac{1-\gamma}{T_0 + 1} } \\ \nonumber
\le & \frac{1}{1-\gamma} (T_0+1)^{1-\gamma}.
\end{align}
$\textcircled{1}$ holds according to Tylor expansion, that is, 
\begin{align}
\nonumber
(1+x)^{a} \le 1 + ax +\frac{a(a-1)}{2}x^2,
\end{align} holds for $-1 < x < 1$ and $-1< a < 0$.

It finally compltes the proof.
\end{proof}

%
%

\bibliographystyle{abbrvnat}
\bibliography{reference}


\end{document}

%% file: yaweinewcomm.tex
\newcommand{\Rcal}{{\mathcal{R}}}
\newcommand{\Acal}{{\mathcal{A}}}
\newcommand{\Fcal}{{\mathcal{F}}}
\renewcommand{\a}{{\bf a}}

\renewcommand{\u}{{\bf u}}
\renewcommand{\v}{{\bf v}}

\newcommand{\x}{{\bf x}}
\newcommand{\y}{{\bf y}}
\newcommand{\z}{{\bf z}}

\newcommand{\Scal}{{\mathcal{S}}}

\newcommand{\Ocal}[1]{{\mathcal{O}\left(#1\right)}}

\newcommand{\Xcal}{{\mathcal{X}}}

\newcommand{\Lcal}{{\mathcal{L}}}

\newcommand{\argmin}{\operatornamewithlimits{argmin}}

\newcommand{\lrincir}[1]{\left( #1 \right)}

\newcommand{\lrnorm}[1]{\left\lVert#1\right\rVert}
\newcommand{\lrangle}[1]{\left\langle#1 \right\rangle}

\newcommand{\EE}{\mathop{\mathbb{E}}}
\newcommand{\RR}{\mathbb{R}}

\newcommand{\refabovecir}[2]{\displaystyle_{#1}^{#2}}